\documentclass{article}

\usepackage[dblblindworkshop, final]{neurips_2025}
\workshoptitle{2nd Workshop on Principles of Generative Modeling (PriGM)}

\usepackage[utf8]{inputenc} 
\usepackage[T1]{fontenc}    
\usepackage{hyperref}       
\usepackage{url}            
\usepackage{booktabs}       
\usepackage{amsmath}        
\usepackage{amssymb}        
\usepackage{amsfonts}       
\usepackage{amsthm}         
\usepackage{nicefrac}       
\usepackage{microtype}      
\usepackage{xcolor}         
\usepackage{graphicx}
\usepackage{multirow}

\newtheorem{theorem}{Theorem}
\newtheorem{proposition}{Proposition}

\title{The Geometry of Language Model Logits: Generic Structure and Learned Departures}

\author{%
  Swathi Shree Narashiman\thanks{Equal contribution. Correspondence: \texttt{narashim@stanford.edu}} \\
  Stanford University
  \And
  Aravind Ramana Venkatramanan\footnotemark[1] \\
  Indian Institute of Technology Madras
  \And
  Nambi Seshadri \\
  University of California, San Diego
}

\begin{document}

\maketitle

\begin{abstract}
  Autoregressive language models produce, at every prediction step, a vocabulary-sized logit vector whose structure determines the conditional next-token distribution. We empirically identify that the ordered vocabulary-head logits follow a remarkably consistent log-linear relation
  \[
    z_{(r)} \approx \mu - \alpha \log r
  \]
  across prediction contexts, architectures, and scales, with the slope developing progressively during pretraining. To our knowledge, this is the first systematic analysis of the logit geometry underlying the distributional assumptions used in prior decoding work such as Mirostat~\citep{basu2021mirostat}. While controlled order-statistic baselines spanning bounded, light-tailed, and heavier-tailed distributions can largely explain the observed log-linear geometry, they do not explain the systematic shift in the slope $\alpha$: trained language models exhibit $\alpha \approx 1.4$--$1.7$, well above the initialization baseline ($\alpha \approx 0.2$). Moreover, $\alpha$ correlates with predictive performance, suggesting that it captures learned structure beyond generic extreme-value effects.
\end{abstract}

\section{Introduction}

At every prediction step, an autoregressive language model assigns a score to every token in its vocabulary, which governs how probability mass is distributed over possible continuations. Sampling is the last step of generation; the logits are where the real decision is already made. While substantial work has studied the resulting token probabilities, generated text, and hidden representations, the structure of this vocabulary-wide logit vector remains comparatively unexplored.

One natural way to expose this structure is to ignore token identities and examine the logits in rank order. This perspective is also directly relevant to common decoding procedures, such as top-$k$ and top-$p$ sampling, which operate on the vocabulary head by selecting or truncating the highest-ranked logits. This leads to a simple question: \textbf{\emph{Does the ordered logit vector possess a simple geometric structure?}}

We empirically find a remarkably regular structure in the vocabulary head: ordered logits follow an approximately log-linear relation across prediction contexts, architectures, and model scales, and its slope is shaped during pretraining. This raises the question of how much of this structure is attributable to language models and how much arises generically from order statistics.

Controlled order-statistic baselines show that ranking any finite collection of random variables induces its own structure, and light-tailed distributions can produce approximately log-linear ordered spectra. We therefore focus on the spectral slope $\alpha$, examining how it differs from generic baselines, how it varies across models, and what information it carries about the model's predictive state.

Our contributions are:
\textbf{(i)} We characterize a consistent log-linear geometry in the ordered
vocabulary-head logits of autoregressive language models across contexts,
architectures, and scales, and trace how it develops during pretraining.
\textbf{(ii)} We distinguish generic order-statistic geometry from learned
structure using controlled distributional baselines, identifying the spectral
slope $\alpha$ as the key empirical distinction between the regimes.
\textbf{(iii)} We show that the ranked next-token distribution can be efficiently reconstructed from the vocabulary head, and that $\alpha$ correlates with \emph{pre-hoc} predictive
uncertainty (entropy) and \emph{post-hoc} prediction quality, measured by NLL,
ground-truth token probability $p_{\mathrm{true}}$, and ground-truth token
rank.

\section{Preliminaries}

\paragraph{Autoregressive prediction.}
At each prediction step, a decoder-only Transformer
\citep{vaswani2017attention} maps the preceding context to a hidden state
$h_k\in\mathbb{R}^{d}$ and produces a vocabulary-sized logit vector
$z_k=W_u h_k+b\in\mathbb{R}^{V}$. The autoregressive factorization and
conditional next-token distribution are
\[
P(t_{1:m})=\prod_{k=1}^{m}P(t_k\mid t_{1:k-1})
\qquad\qquad
P(t_k=v\mid t_{1:k-1})
=
\frac{\exp(z_{k,v})}
{\sum_{v'}\exp(z_{k,v'})}.
\]
Here, $W_u\in\mathbb{R}^{V\times d}$ is the unembedding matrix and
$b\in\mathbb{R}^{V}$ is the output bias. Inference-time sampling may further
rescale $z_k$ by a temperature $T$ before softmax \citep{ackley1985}.

\paragraph{Rank-space representation.}
Because token identities are not relevant to the structural question we aim to
answer, we study the vocabulary head in rank space by sorting the logits in
descending order:
\[
z_{(1)}\ge z_{(2)}\ge\cdots\ge z_{(V)},
\qquad
z_{(r)}:=\textit{$r$-th largest entry of }z_k.
\]
For a chosen head size $K\le V$, we analyze the ordered spectrum
$\{z_{(r)}\}_{r=1}^{K}$.

\paragraph{Relation to Zipfian decoding.}
Our analysis is closely related to the Zipfian assumption underlying Mirostat
\citep{basu2021mirostat}, but addresses a more fundamental structural question.
While Mirostat exploits a power-law next-token distribution for adaptive
decoding, we characterize the ordered logit spectrum before softmax, ask
whether its structure is generic or learned, and show that its two-parameter
law can approximately reconstruct the ranked next-token distribution from a
small exact vocabulary head. A detailed comparison is provided in
Appendix~\ref{appendix:related-work}.

\section{Results}
\label{sec:results}
\subsection{Log-Linear Logit Geometry}

\begin{figure}[b]
    \centering
    \begin{minipage}[t]{0.32\textwidth}
        \centering
        \includegraphics[width=\linewidth]{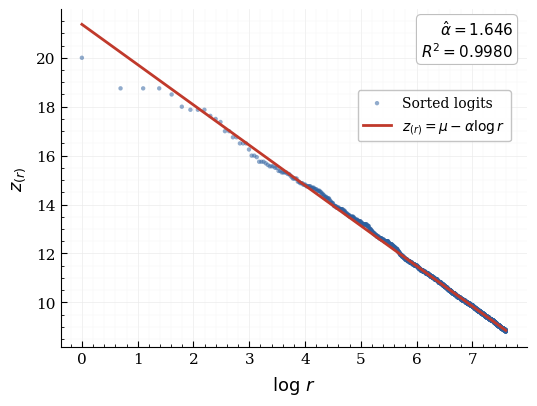}
    \end{minipage}
    \hfill
    \begin{minipage}[t]{0.32\textwidth}
        \centering
        \includegraphics[width=\linewidth]{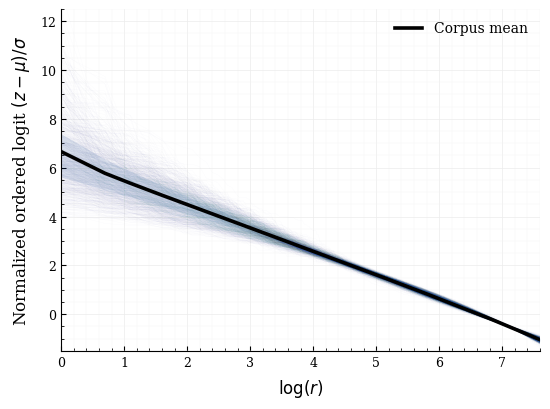}
    \end{minipage}
    \hfill
    \begin{minipage}[t]{0.32\textwidth}
        \centering
        \includegraphics[width=\linewidth]{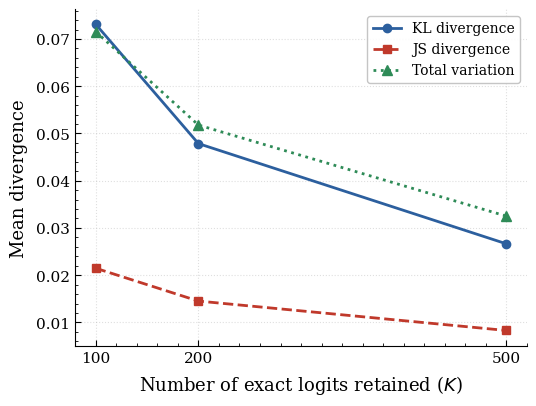}
    \end{minipage}

    \caption{
    \textbf{Ordered logit geometry.}
    (a) Ordered logit spectrum for a representative prediction context from
    Llama-3.2-1B, with logits sorted in descending order and plotted against
    $\log r$. The red line shows the least-squares fit of
    $z_{(r)}=\mu-\alpha\log r$ over the top $K=2000$ vocabulary ranks.
    (b) Collapse of normalized ordered logit spectra across 1,000 randomly
    sampled prediction contexts from ten books for Llama-3.2-1B. Each
    spectrum is normalized independently as $(z-\bar z)/\sigma$, and the black
    curve shows the corpus mean.
    (c) Mean divergence between the true next-token distribution and its
    spectral reconstruction using $K$ exact logits and a two-parameter
    log-linear approximation of the remaining ranks, averaged over four
    language models.
    Experimental details and additional results are provided in
    Appendix~\ref{appendix:experimental-details}.
    }
    \label{fig:logit_geometry}
\end{figure}

For a prediction context, let $z_{(1)} \ge \cdots \ge z_{(V)}$ denote the vocabulary logits sorted in descending order. We empirically observe that the vocabulary head is well approximated by
\[
z_{(r)} \approx \mu - \alpha \log r,
\qquad r=1,\ldots,K,
\]
where $\mu$ is a context-dependent intercept and $\alpha>0$ is the spectral slope. Across the models and contexts studied, this relation provides an excellent fit over the vocabulary head, with mean $R^2\approx0.99$ at $K=2000$ for every model (Table~\ref{tab:model_summary}).

\begin{figure}[b]
    \centering
    \includegraphics[width=\linewidth]{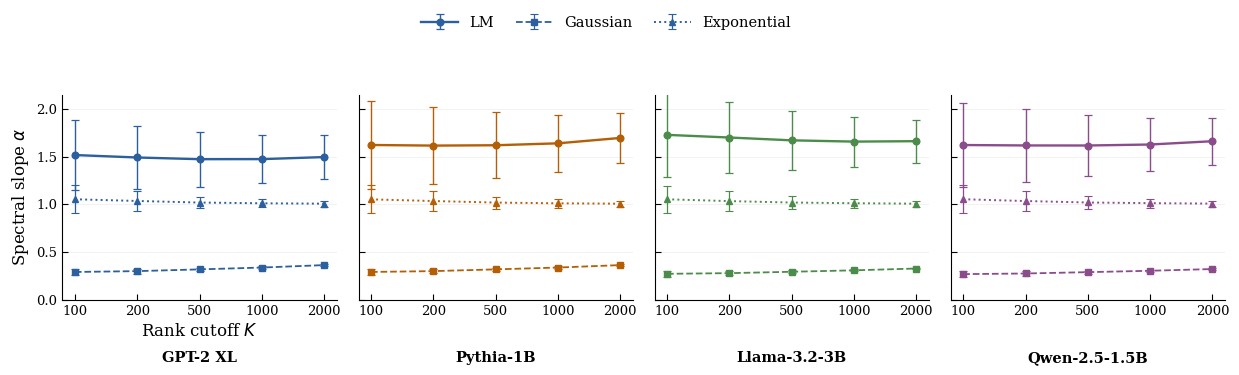}
    \caption{
    \textbf{Rank-cutoff robustness of the learned spectral regime.}
    Spectral slope $\alpha$ as a function of rank cutoff $K$ for representative
    models from four model families. Each language model is compared with
    Gaussian and Exponential order-statistic baselines matched to its vocabulary
    size. Error bars denote standard deviations across the 10 books for language
    models and across synthetic samples for the null distributions. The corresponding dependence of goodness of fit $R^2$ on $K$ is reported in
    Figure~\ref{fig:r2_k_dependence} (Appendix~\ref{appendix:k-dependence}).
    }
    \label{fig:k_dependence}
\end{figure}

Figure~\ref{fig:logit_geometry}(a) illustrates the fit for a representative
context, and Figure~\ref{fig:logit_geometry}(b) shows that the log-linear
structure persists across substantially different prediction contexts, rather
than being tied to a particular sequence. After normalization, the spectra
collapse onto a common curve (Figure~\ref{fig:logit_geometry}(b)); since
$z$-scoring maps any exactly log-linear spectrum to the same curve, the
tightness of this collapse reflects the per-context quality of the log-linear
fit (Appendix~\ref{subsec:collapse-appendix}). Robustness across $K$, models,
and corpora is reported in Appendices~\ref{appendix:ablations}
and~\ref{appendix:k-dependence}.

\subsection{Generic vs.\ Learned Structure}

The near-perfect fit of ordered logits does not by itself imply
language-model-specific structure. We therefore apply the same fitting
procedure to controlled order-statistic baselines spanning bounded,
light-tailed, and heavier-tailed distributions. Uniform samples fit
substantially worse, with average $R^2\approx0.75$ across vocabulary sizes and
$\alpha\approx0.001$--$0.020$, while Gaussian and Exponential samples achieve
$R^2\approx0.99$ and $R^2\approx0.97$, respectively. Lognormal samples achieve
$R^2\approx0.92$ but produce much larger slopes, $\alpha\approx4.67$--$10.6$.
Thus, although log-linear structure can arise generically from order
statistics, its fit and associated slope depend substantially on the
underlying distribution.

Figure~\ref{fig:k_dependence} shows that the separation between language
models and these baselines is not specific to the choice of rank cutoff.
Across a broad range of $K$, language models exhibit stable slopes, while the
order-statistic baselines display distribution-dependent behavior.

\begin{figure*}[t]
  \centering
  \includegraphics[width=\textwidth]{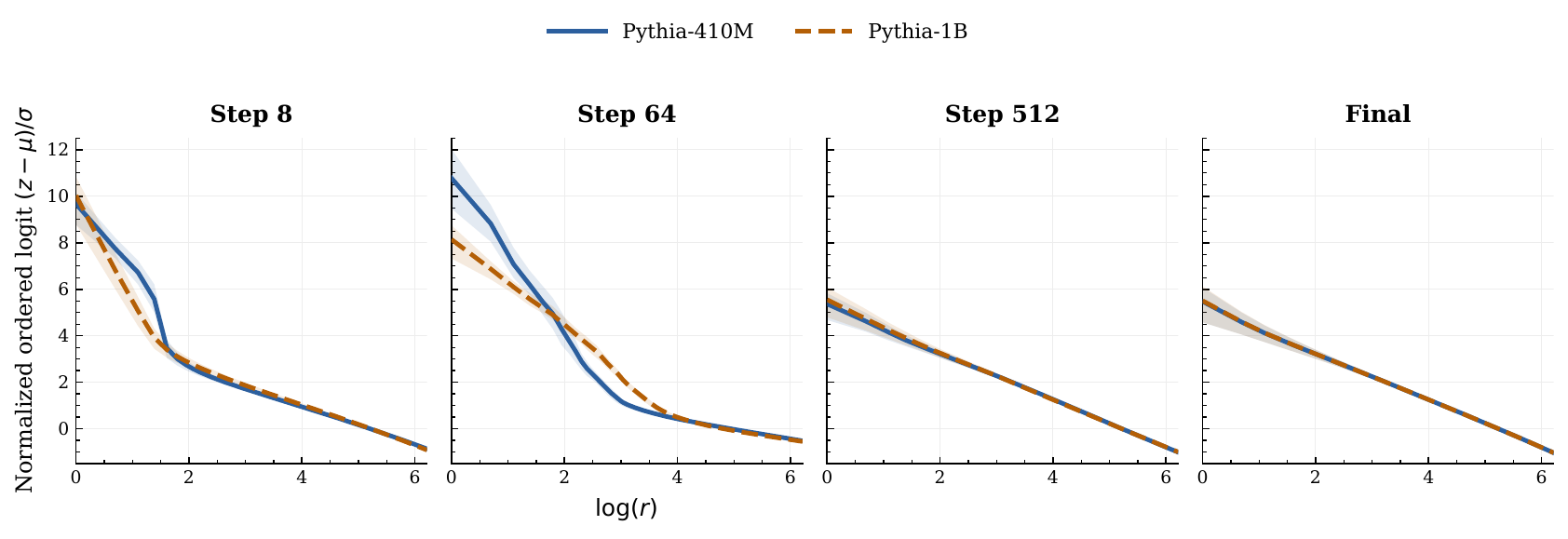}
  \caption{
  \textbf{Evolution of ordered logit geometry during pretraining.}
  Mean normalized ordered-logit spectra for Pythia-410M (solid) and Pythia-1B
  (dashed) across representative training checkpoints. Shaded regions denote
  the interquartile range across prediction contexts.
  }
  \label{fig:training_dynamics}
\end{figure*}

\textbf{What shapes the spectral slope $\alpha$?}
The cross-entropy objective provides a partial answer. At the population
optimum, $z_i(h)=\log q(i\mid h)+c(h)$, making the log-linear logit law
equivalent to an approximately Zipfian conditional distribution
$q(\cdot\mid h)$---the same regularity estimated online by
Mirostat~\citep{basu2021mirostat}. This equivalence does not itself explain
why language-model distributions should be Zipfian: it holds exactly only at
the unconstrained population optimum. Rather, we view the empirical tightness
and stability of $\alpha$ across architectures and scales as the phenomenon
requiring explanation. Further equivalences and consequences are given in
Appendix~\ref{appendix:theory-proofs}.

\textbf{Evolution during pretraining.}
Publicly available Pythia checkpoints show that the spectral regime develops
during training. Although the initialized model is already approximately
log-linear, $\alpha$ increases from approximately $0.2$ at initialization,
while normalized spectra progressively approach the final-model shape, with
most changes occurring early and the structure clearly established by step 512
(Figure~\ref{fig:training_dynamics}). Training therefore systematically shapes
the spectral parameters within an ordered-logit structure. Additional analyses
of the training dynamics are provided in Appendix~\ref{appendix:training-dynamics}.

\subsection{Spectral Reconstruction of Next-Token Distributions}
\label{subsec:predictive-behavior}

The log-linear structure also provides a compact approximation of the
next-token distribution. Given only the top-$K$ ordered logits, we retain
their exact values and fit the spectral parameters $(\mu,\alpha)$ by least
squares. The remaining vocabulary logits are then reconstructed according
to $\hat z_{(r)}=\hat\mu-\hat\alpha\log r$, yielding an approximate
distribution after softmax.

\textbf{Distribution reconstruction.}
Despite reconstructing the majority of vocabulary logits using only the
two fitted spectral parameters, the resulting distributions remain close
to the exact softmax distribution, with divergence consistently decreasing
as the number of retained exact logits $K$ increases
(Figure~\ref{fig:logit_geometry}(c)).
At $K=500$ ($\leq 1\%$ of the vocabulary), the reconstruction achieves
a mean Jensen--Shannon divergence of only $0.008$ across four model
families. Thus, a small exact head together with a two-parameter spectral
description of the remaining tail approximately recovers the full
rank-space next-token distribution (Appendix~\ref{appendix:spectral-approximation}).

\textbf{Relation to predictive state.}
The spectral slope $\alpha$ also varies systematically with the model's
predictive state. Across 40 model--book pairs, larger $\alpha$ is associated
with lower uncertainty and better prediction of the realized next token:
entropy correlates at $r=-0.332$, NLL and ground-truth log-rank at
$r=-0.206$ and $r=-0.209$, and $p_{\mathrm{true}}$ at $r=+0.111$.
These associations are highly consistent across pairs and persist without
$R^2$ filtering and under Spearman correlations, suggesting that $\alpha$
captures variation in the predictive state beyond ordered logit geometry
(Appendix~\ref{appendix:predictive-analysis}).

\section{Discussion}

Our results reveal two complementary aspects of the ordered logit spectrum.
Approximately log-linear geometry can arise generically from order statistics,
while training systematically shapes the spectrum through $\alpha$. This yields a compact rank-space representation of
vocabulary-sized predictions: a small exact head and two fitted parameters can
approximately reconstruct the ranked
probability spectrum. This compact structure may provide a basis for approximating vocabulary-sized predictive distributions without
retaining the full spectrum. Moreover,
$\alpha$ systematically tracks aspects of the model's predictive state, providing a compact summary of variation in next-token
predictions.

\paragraph{Limitations and future work.}
Our analysis focuses on pretrained decoder-only language models and provides
sufficient rather than necessary theoretical conditions. The origin of the
learned spectral regime, and in particular why optimization drives the spectrum
toward its observed parameter range, remains an open question. Moreover,
our reconstruction concerns the ranked probability spectrum rather than token
identities in the tail. Future work
should examine whether these dynamics extend to other model architectures and
whether the spectral functional can support uncertainty estimation,
generation-state tracking, distribution approximation, or decoding control.

\newpage
\bibliographystyle{plainnat}
\bibliography{custom}

\newpage
\appendix

\section{Related Work}
\label{appendix:related-work}

\paragraph{Statistical laws of language and generated text.}
Zipf's law describes corpus-level token frequencies aggregated over natural
language \citep{zipf1949,mandelbrot1953,piantadosi2014}. More recently,
several works have investigated whether neural language models reproduce such
statistical laws in generated text \citep{takahashi2019,meister2021}. Our
object of study is different: we analyze the model's predicted logit vector
for each individual prediction context, before any token is sampled. The
resulting log-linear law is therefore a \emph{pre-hoc} property of conditional
predictions rather than a post-hoc property of natural or generated text.

\paragraph{Structure of the conditional next-token distribution.}
Several works exploit regularities in the conditional next-token distribution
to design decoding algorithms. Most closely related is Mirostat
\citep{basu2021mirostat}, which assumes Zipfian statistics of the predicted
distribution and estimates the corresponding exponent online to adapt the
sampling cutoff and control perplexity. Our work asks a prior structural
question: \emph{what structure underlies such a power-law form, and how much
of it is learned?} Rather than positing Zipfian structure at the probability
level for decoding, we directly study the ordered logits before softmax and
find that they follow a remarkably tight log-linear law across individual
prediction contexts, architectures, and scales.

This perspective also separates generic structure induced by sorting from
structure learned by the language model. We characterize the observed spectral
slope $\alpha$, study its variation across predictive states, show that the
slope develops progressively during pretraining, and demonstrate that a
small exact vocabulary head together with the fitted spectral tail can
approximately reconstruct the ranked next-token distribution. At the
cross-entropy optimum, our logit law is exactly equivalent to a Zipfian
conditional distribution, establishing a direct theoretical connection to the
probability-level assumption used by Mirostat.

\paragraph{Geometry of the output layer.}
A separate line of work studies geometric properties of transformer outputs,
including the softmax bottleneck \citep{yang2018}, stolen probability
\citep{demeter2020}, representation degeneration \citep{gao2019}, anisotropy
\citep{ethayarajh2019}, and low-rank structure in collections of logit vectors
across prompts \citep{finlayson2024logits}. These perspectives are complementary: while prior work
primarily studies structure across contexts or representations, we identify an
interpretable \emph{within-context} geometry of the vocabulary-sized output
itself. The resulting two-parameter description provides a compact descriptor
of each prediction step.

\paragraph{Training dynamics and theoretical characterization.}
Prior work has largely focused on analyzing trained models. We additionally
study the emergence of ordered logit geometry during pretraining using
intermediate checkpoints from the Pythia suite \citep{biderman2023pythia},
showing that the spectral regime develops progressively through optimization. On the
theoretical side, rather than assuming a particular generative model for the
observed spectrum, we characterize its relation to conditional Zipfian
structure at the cross-entropy optimum, give an equivalent distance-based
reparametrization, and establish deterministic consequences for tail counts
and the softmax distribution.

\section{Experimental Setup}
\label{appendix:experimental-details}
\subsection{Models}

We evaluate four families of decoder-only autoregressive language models spanning a range of model scales, vocabularies, and training procedures: GPT-2 (124M, 355M, 774M, 1.5B) \citep{radford2019language}, Pythia (410M, 1B) \citep{biderman2023pythia}, Llama-3.2 (1B, 3B) \citep{llama3_2024}, and Qwen2.5 (0.5B, 1.5B) \citep{qwen2025}. We refer to individual models as \emph{Family-Size} (e.g., GPT-2-1.5B, i.e., GPT-2 XL). To study the development of the ordered logit geometry during training, we additionally evaluate intermediate checkpoints of the Pythia models at optimization steps $\{0, 8, 64, 512, 8000, 143000\}$. All models are evaluated in inference mode without gradient computation.

\subsection{Corpora}

Experiments are conducted on ten books from Project Gutenberg\footnote{\url{https://www.gutenberg.org}} covering diverse genres including realism, science fiction, mystery, and adventure. Each text is tokenized using the model's native tokenizer. We use the first 20,000 tokens of each book with a context length of 128 tokens and stride 1, yielding approximately $2\times10^4$ prediction contexts per book. For every prediction context, we store the top-$K$ sorted logits ($K=2000$). Code is available at \url{https://github.com/USERNAME/Log_linear_logit_law}. 
As an ablation, we additionally verify the normalized collapse for GPT-2 on two other corpora (WikiText-2 and French Wikipedia), reported in Appendix~\ref{appendix:ablations}.

\subsection{Fitting Procedure}
\label{appendix:fitting-procedure}
For each prediction context, we fit the log-linear model
\[
z_{(r)}=\mu-\alpha\log r,
\qquad r=1,\ldots,K,
\]
using ordinary least squares. We report the fitted slope $\hat{\alpha}$, intercept $\hat{\mu}$, coefficient of determination ($R^2$), and root mean squared error (RMSE). Per-model statistics are reported as the mean of per-book means, with standard deviation across books as a measure of corpus stability. For the normalized spectrum analysis, each logit spectrum is standardized as
\[
\tilde{z}_{(r)}=\frac{z_{(r)}-\bar{z}}{\sigma},
\]
before averaging across contexts.

\section{Additional Experiments}
\label{sec:ablation}

\subsection{Functional Form Selection}
\label{appendix:functional-form}

We compare the log-linear form $z_{(r)} = \mu - \alpha \log r$ against five
alternatives on prediction contexts from Llama-3.2-1B; results are in
Table~\ref{tab:functional_form}.

\begin{table}[h]
\centering
\small
\setlength{\tabcolsep}{5pt}
\renewcommand{\arraystretch}{1.15}
\caption{Comparison of functional forms fitted to ordered logit spectra
(Llama-3.2-1B, $K=2000$). \# is the number of free parameters.}
\label{tab:functional_form}
\begin{tabular}{@{}lccc@{}}
\toprule
\textbf{Form} & \textbf{\#} & $R^2$ & \textbf{RMSE} \\
\midrule
Logarithmic & 2 & $0.9922 \pm 0.0074$ & $0.1299 \pm 0.0588$ \\
Power       & 3 & $0.9962 \pm 0.0045$ & $0.0874 \pm 0.0437$ \\
Exponential & 2 & $0.9460 \pm 0.0290$ & $0.3678 \pm 0.0849$ \\
Square-root & 2 & $0.9043 \pm 0.0391$ & $0.4941 \pm 0.1298$ \\
Linear      & 2 & $0.7752 \pm 0.0534$ & $0.7692 \pm 0.1464$ \\
Inverse     & 2 & $0.1950 \pm 0.0409$ & $1.4642 \pm 0.2010$ \\
\bottomrule
\end{tabular}
\end{table}

The three-parameter power law $z_{(r)} = a \cdot r^{-b} + c$ is essentially
matched by the two-parameter log-linear form, with its small residual gain
in fit attributable to the extra parameter rather than a qualitatively
better description. We adopt the log-linear form for two reasons: (1)
\textbf{parsimony}: it reaches $R^2 > 0.99$ with two parameters, one
fewer than the power law; and (2) \textbf{softmax tractability}: it
yields a closed-form Zipf law $p_r \propto r^{-\alpha}$, whereas
the power law admits no comparably tractable softmax expression. We do
not claim log-linear is uniquely correct, only that it is the simplest
form consistent with the data and better suited to the softmax-based
results in Section~\ref{sec:results}.

\subsection{Ablations}
\label{appendix:ablations}

\begin{figure}[t]
    \centering
    \begin{minipage}[t]{0.48\columnwidth}
        \centering
        \includegraphics[width=\linewidth]{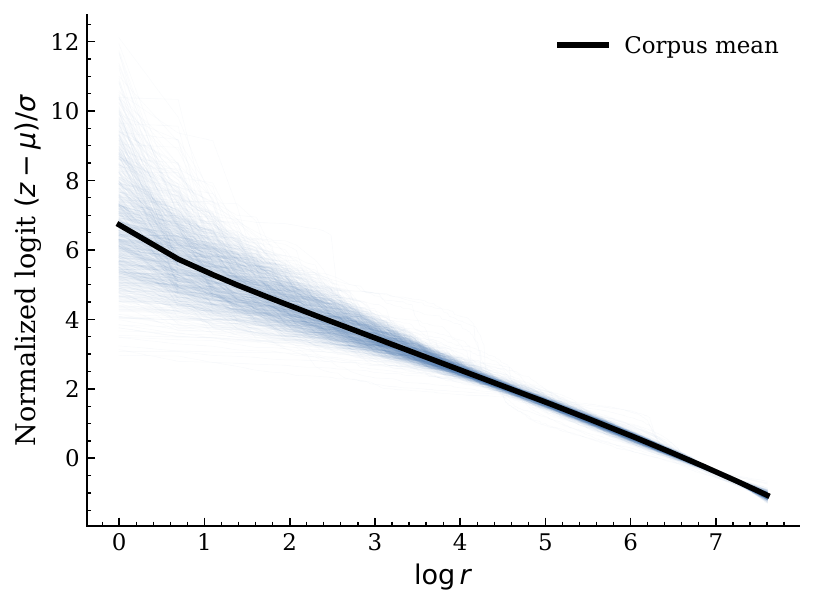}
    \end{minipage}
    \hfill
    \begin{minipage}[t]{0.48\columnwidth}
        \centering
        \includegraphics[width=\linewidth]{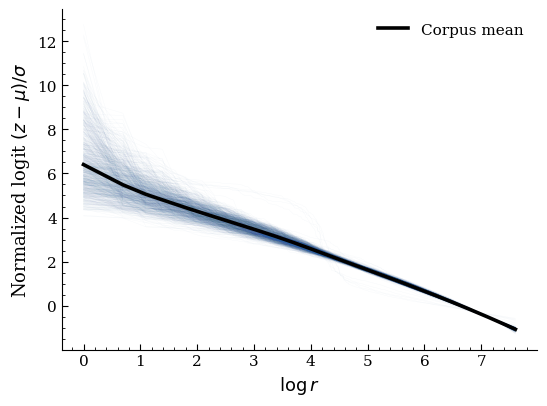}
    \end{minipage}

    \caption{
    \textbf{Cross-corpus normalized logit collapse for GPT-2.}
    \textbf{Left:} GPT-2 evaluated on WikiText-2.
    \textbf{Right:} GPT-2 evaluated on French Wikipedia.
    In both panels, blue curves show individual prediction contexts and the
    black curve shows the corpus mean. The normalized ordered-logit spectra
    $(z-\bar z)/\sigma$ exhibit a shared trajectory across contexts in both
    corpora.
    }
    \label{fig:gpt2_corpus_ablation}
\end{figure}

To verify that the rank-logit collapse is not specific to the primary model or
corpus, we repeat the analysis with GPT-2 on WikiText-2 and French Wikipedia.
As shown in Figure~\ref{fig:gpt2_corpus_ablation}, the normalized logit curves
$(z-\bar z)/\sigma$ collapse onto a shared trajectory across contexts in both
corpora, consistent with the pattern observed in our primary experiments.

\subsection{Additional Results}
\label{appendix:k-dependence}

\begin{table}[h]
\centering
\small
\setlength{\tabcolsep}{6pt}
\renewcommand{\arraystretch}{1.15}
\caption{Log-linear fit statistics across autoregressive language model families, averaged over 10 books ($K=2000$). Reported $\alpha$ is mean $\pm$ standard deviation across books;
$R^2$ is the mean per-context coefficient of determination, averaged across books.}
\label{tab:model_summary}
\begin{tabular}{@{}llcc@{}}
\toprule
\textbf{Family} & \textbf{Model} & $\alpha$ & $R^2$ \\
\midrule
\multirow{4}{*}{GPT-2}
 & 124M  & $1.426 \pm 0.243$ & 0.9910 \\
 & 355M  & $1.413 \pm 0.213$ & 0.9914 \\
 & 774M  & $1.488 \pm 0.220$ & 0.9905 \\
 & 1.5B  & $1.494 \pm 0.229$ & 0.9898 \\
\midrule
\multirow{2}{*}{Pythia}
 & 410M  & $1.649 \pm 0.260$ & 0.9889 \\
 & 1B    & $1.695 \pm 0.266$ & 0.9889 \\
\midrule
\multirow{2}{*}{Llama-3.2}
 & 1B    & $1.629 \pm 0.234$ & 0.9915 \\
 & 3B    & $1.661 \pm 0.224$ & 0.9908 \\
\midrule
\multirow{2}{*}{Qwen2.5}
 & 0.5B  & $1.558 \pm 0.229$ & 0.9918 \\
 & 1.5B  & $1.660 \pm 0.247$ & 0.9917 \\
\bottomrule
\end{tabular}
\end{table}

\begin{figure}[t]
    \centering
    \includegraphics[width=\linewidth]{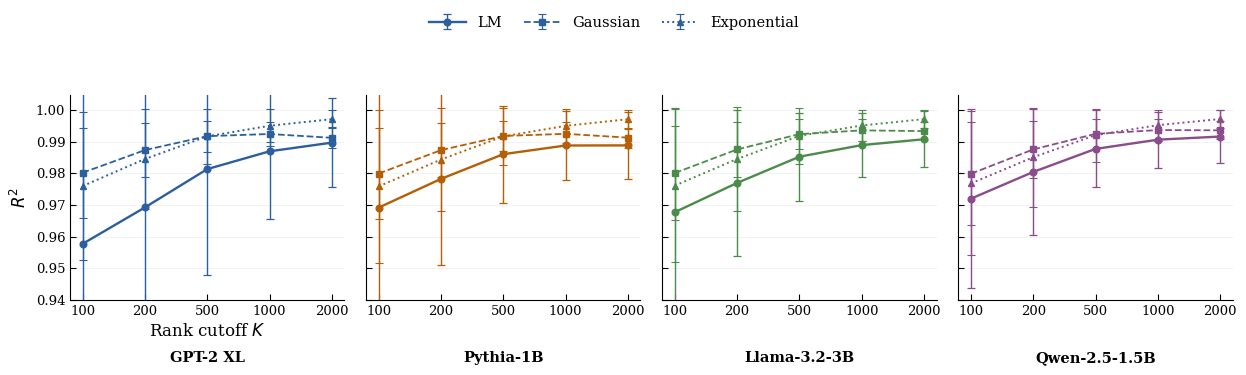}
    \caption{
    \textbf{Dependence of goodness of fit on rank cutoff.}
    Coefficient of determination $R^2$ as a function of rank cutoff $K$ for the
    same representative language models and vocabulary-matched Gaussian and
    Exponential order-statistic baselines as in Figure~\ref{fig:k_dependence}.
    Error bars denote standard deviations across the 10 books for language
    models and across synthetic samples for the null distributions.
    }
    \label{fig:r2_k_dependence}
\end{figure}

\paragraph{Consistency across model families.}
Table~\ref{tab:model_summary} summarizes the fitted log-linear parameters across all evaluated models. Despite substantial differences in architecture, tokenizer, training corpus, and more than an order-of-magnitude variation in model size, the ordered logit spectrum is consistently well approximated by a log-linear law (mean $R^2 \geq 0.988$ for every model). More strikingly, the fitted slope remains confined to a narrow range ($1.41 \leq \alpha \leq 1.70$), indicating that modern autoregressive language models exhibit remarkably similar ordered logit geometry despite their diverse training procedures and scales.

\paragraph{Dependence of goodness of fit on $K$.}
The goodness of fit generally improves as the rank cutoff increases for both
language models and order-statistic baselines. Language-model fits approach
$R^2\approx0.99$ at larger $K$, but Gaussian and Exponential baselines attain
similarly high or higher $R^2$ over the same range. Thus, the improvement in
$R^2$ with $K$ is not specific to language models and should not itself be
interpreted as evidence of learned structure. Instead, we use the spectral
slope $\alpha$ to distinguish the language-model spectra from generic
order-statistic geometry.

\section{Additional Training Dynamics}
\label{appendix:training-dynamics}

\paragraph{Evolution of the spectral regime.}
We further quantify the evolution of the ordered-logit law throughout
pretraining using intermediate checkpoints from the Pythia-410M and
Pythia-1B suites. Table~\ref{tab:training-dynamics} reports the fitted
spectral exponent $\alpha$ and goodness of fit at different training stages.

Both models exhibit a substantial change in the spectral regime during the
early stages of optimization. At initialization (Step~0), the ordered spectra
already admit a highly linear fit, but occupy a shallow regime, with
$\alpha \approx 0.20$ at $K=500$, closer to the Gaussian synthetic baseline
($\alpha \approx 0.3$) than to the trained models. This exponent is
substantially smaller than that of the trained models, whose final values
reach $\alpha \approx 1.57$--$1.62$ at the same cutoff.

\begin{table}[htbp]
\centering
\small
\caption{Evolution of the ordered-logit spectral regime during pretraining.
We report the mean fitted exponent $\alpha$ and $R^2$ at $K=500$ across
evaluation contexts for Pythia-410M and Pythia-1B.}
\label{tab:training-dynamics}
\begin{tabular}{llcccc}
\toprule
& & \multicolumn{2}{c}{Pythia-410M}
& \multicolumn{2}{c}{Pythia-1B} \\
\cmidrule(lr){3-4}
\cmidrule(lr){5-6}
Checkpoint & Training stage
& $\alpha$ & $R^2$
& $\alpha$ & $R^2$ \\
\midrule
Step 0      & Initialization & 0.203 & 0.992 & 0.199 & 0.991 \\
Step 8      & Early          & 0.258 & 0.881 & 0.241 & 0.909 \\
Step 64     & Early          & 0.507 & 0.680 & 0.779 & 0.805 \\
Step 512    & Intermediate   & 1.128 & 0.990 & 1.164 & 0.988 \\
Step 8000   & Late           & 1.485 & 0.987 & 1.498 & 0.987 \\
Step 143000 & Final          & 1.571 & 0.987 & 1.618 & 0.986 \\
\bottomrule
\end{tabular}
\end{table}

In the earliest stages of training, the log-linear structure itself undergoes
a substantial departure from the final regime (Figure~\ref{fig:training_dynamics}). For example, at Step~64, the
$K=500$ fit reaches only $R^2=0.68$ and $0.81$ for Pythia-410M and
Pythia-1B, respectively. As training progresses, both the exponent and
goodness of fit move toward their final values. By Step~512, the $K=500$
exponent reaches $1.13$ for Pythia-410M and $1.16$ for Pythia-1B, with
$R^2 \approx 0.99$. At Step~8000, both models are already close to their
final spectral regime, with $\alpha=1.48$--$1.50$ and
$R^2\approx0.987$. The final checkpoints further sharpen this regime,
reaching $\alpha=1.57$ for Pythia-410M and $1.62$ for Pythia-1B.

These dynamics indicate that training does not merely preserve or amplify a
generic consequence of ordering. Rather, optimization transforms the initial
shallow spectrum through a transient non-log-linear phase into a distinct
spectral regime characterized by a substantially larger and increasingly
stable exponent. The learned phenomenon is therefore characterized jointly by
the re-emergence of strong log-linearity and a nontrivial spectral scale that is
not present in either synthetic or untrained baselines.

\section{Spectral Approximation of Next-Token Distributions}
\label{appendix:spectral-approximation}

We provide additional details and model-wise results for the spectral
reconstruction experiment in Section~\ref{subsec:predictive-behavior}.
The experiment evaluates whether the ordered logit spectrum can be used to
approximately recover the ranked next-token distribution from a small exact
head and a two-parameter description of the remaining ranks.

\subsection{Spectral Reconstruction}

For each prediction context, let
$z_{(1)} \geq z_{(2)} \geq \cdots \geq z_{(V)}$
denote the vocabulary logits sorted in descending order. Given the first
$K$ exact ordered logits, we fit
\[
z_{(r)} \approx \mu - \alpha \log r
\]
by ordinary least squares over $r=1,\ldots,K$. We retain the first $K$
logits exactly and reconstruct the remaining ranks as
\[
\hat z_{(r)} =
\begin{cases}
z_{(r)}, & r \leq K,\\
\hat\mu-\hat\alpha\log r, & r>K.
\end{cases}
\]
Applying softmax to the reconstructed spectrum yields the approximate
distribution $\hat p$, which is compared with the exact distribution
$p=\mathrm{softmax}(z)$.

We evaluate reconstruction using Kullback--Leibler (KL) divergence,
Jensen--Shannon (JS) divergence, and total variation (TV) distance:
\[
D_{\mathrm{KL}}(p\Vert\hat p), \qquad
D_{\mathrm{JS}}(p,\hat p), \qquad
D_{\mathrm{TV}}(p,\hat p).
\]

Results are aggregated over 39,520 prediction contexts (9,880 per model) sampled from ten
books across four language models. We consider $K\in\{100,200,500\}$.

\subsection{Model-wise Results}

Table~\ref{tab:spectral-reconstruction-modelwise} reports the complete
model-wise reconstruction results. Divergence decreases consistently as
more exact logits are retained. At $K=500$, mean JS divergence ranges from
$0.0058$ to $0.0094$ across models, while mean TV ranges from $0.0235$ to
$0.0369$.

\begin{table*}[htbp]
\centering
\small
\caption{Model-wise spectral reconstruction results. Divergences compare
the exact next-token distribution with the distribution reconstructed from
the top-$K$ exact ordered logits and a log-linear approximation of the
remaining ranks.}
\label{tab:spectral-reconstruction-modelwise}
\begin{tabular}{llrrrrr}
\toprule
Model & $K$ & $N$ & Mean $R^2$ & Mean KL & Mean JS & Mean TV \\
\midrule
Pythia-1B & 100 & 9880 & 0.9689 & 0.0805 & 0.0238 & 0.0772 \\
          & 200 & 9880 & 0.9781 & 0.0536 & 0.0163 & 0.0566 \\
          & 500 & 9880 & 0.9860 & 0.0292 & 0.0091 & 0.0348 \\
\midrule
Qwen2.5-1.5B & 100 & 9880 & 0.9721 & 0.0848 & 0.0249 & 0.0786 \\
             & 200 & 9880 & 0.9803 & 0.0548 & 0.0167 & 0.0562 \\
             & 500 & 9880 & 0.9874 & 0.0300 & 0.0094 & 0.0347 \\
\midrule
GPT-2-1.5B & 100 & 9880 & 0.9568 & 0.0675 & 0.0197 & 0.0717 \\
           & 200 & 9880 & 0.9686 & 0.0468 & 0.0141 & 0.0543 \\
           & 500 & 9880 & 0.9811 & 0.0285 & 0.0088 & 0.0369 \\
\midrule
Llama-3.2-3B & 100 & 9880 & 0.9666 & 0.0595 & 0.0175 & 0.0587 \\
             & 200 & 9880 & 0.9760 & 0.0363 & 0.0110 & 0.0400 \\
             & 500 & 9880 & 0.9847 & 0.0189 & 0.0058 & 0.0235 \\
\bottomrule
\end{tabular}
\end{table*}

\subsection{Aggregate Results}

Across all four models, reconstruction quality improves consistently with
$K$. The mean KL divergence decreases from $0.0731$ at $K=100$ to $0.0479$
at $K=200$ and $0.0266$ at $K=500$. Mean JS divergence decreases from
$0.0215$ to $0.0145$ and $0.0083$, while mean TV decreases from $0.0715$ to
$0.0518$ and $0.0325$, respectively (Table~\ref{tab:spectral-reconstruction-overall}).

\begin{table}[t]
\centering
\small
\caption{Aggregate spectral reconstruction results across 39,520 contexts
and four language models.}
\label{tab:spectral-reconstruction-overall}
\begin{tabular}{rrrrr}
\toprule
$K$ & Mean KL & Mean JS & Mean TV & Mean $R^2$ \\
\midrule
100 & 0.0731 & 0.0215 & 0.0715 & 0.9661 \\
200 & 0.0479 & 0.0145 & 0.0518 & 0.9758 \\
500 & 0.0266 & 0.0083 & 0.0325 & 0.9848 \\
\bottomrule
\end{tabular}
\end{table}

\section{Predictive-Behavior Analysis}
\label{appendix:predictive-analysis}

\paragraph{Setup.}
For each prediction context, we sort the vocabulary logits and fit
\[
z_{(r)}=\mu-\alpha\log r,\qquad r=1,\ldots,K,
\]
with $K=2000$ by ordinary least squares. We evaluate GPT-2-1.5B, Pythia-1B,
Llama-3.2-3B, and Qwen2.5-1.5B on ten Project Gutenberg books per model.
Contexts have length 128 and are sampled every 128 tokens, yielding
non-overlapping windows. For each context, we record $\alpha$, $R^2$, NLL,
$p_{\mathrm{true}}$, true-token rank, and full-vocabulary entropy
\[
H=-\sum_{v=1}^{V}p_v\log p_v.
\]
We use $\log(\text{true rank})$ for the rank analysis. Pearson and Spearman
correlations between $\alpha$ and each statistic are computed separately
within each of the 40 model--book pairs, then pooled with a
DerSimonian--Laird random-effects meta-analysis on Fisher-$z$ transformed
correlations (sampling variance $1/(N-3)$), giving pooled correlations,
95\% confidence intervals, two-sided $p$-values, and heterogeneity $I^2$.
An exact two-sided binomial sign test over the 40 book-level correlations is
included as a distribution-free check. We report results both with a
fit-quality filter ($R^2\geq0.95$) and without one ($R^2\geq0$).

\paragraph{Correlation results.}
Table~\ref{tab:predictive_full_stats} reports the pooled Pearson and Spearman
correlations between $\alpha$ and each predictive statistic, together with
the sign-test summary.

\begin{table}[htbp]
\centering
\small
\setlength{\tabcolsep}{3.5pt}
\setlength{\belowcaptionskip}{4pt}
\renewcommand{\arraystretch}{1.02}
\caption{
Predictive-behavior analysis across 40 model--book pairs. Entries are
random-effects pooled Pearson and Spearman correlations with 95\%
confidence intervals. The Sign column reports the percentage of
book-level correlations with negative/positive sign. For the Pearson
meta-analysis, $p<10^{-12}$ for all statistics except $p_{\mathrm{true}}$,
for which $p=2.02\times10^{-11}$ at $R^2\geq0.95$ and $p=1.36\times10^{-12}$ at
$R^2\geq0$. Exact two-sided binomial sign-test $p$-values are given in
parentheses.
}
\label{tab:predictive_full_stats}
\begin{tabular}{@{}llccc@{\hspace{6pt}}cc@{}}
\toprule
\textbf{$R^2$} & \textbf{Statistic}
& \multicolumn{2}{c}{\textbf{Pearson}}
& \textbf{Sign}
& \multicolumn{2}{c}{\textbf{Spearman}} \\
\cmidrule(lr){3-4}
\cmidrule(lr){6-7}
& & $r$ & 95\% CI & & $\rho$ & 95\% CI \\
\midrule
\multirow{4}{*}{$\geq0.95$}
& Entropy
& $-0.332$ & $[-0.370,-0.293]$
& $\mathbf{-100\%}$ $(1.82{\times}10^{-12})$
& $-0.290$ & $[-0.331,-0.248]$ \\
& NLL
& $-0.206$ & $[-0.236,-0.175]$
& $\mathbf{-97.5\%}$ $(7.46{\times}10^{-11})$
& $-0.159$ & $[-0.193,-0.125]$ \\
& $p_{\mathrm{true}}$
& $+0.111$ & $[+0.079,+0.143]$
& $\mathbf{+87.5\%}$ $(1.38{\times}10^{-6})$
& $+0.159$ & $[+0.125,+0.193]$ \\
& Log-rank
& $-0.209$ & $[-0.236,-0.181]$
& $\mathbf{-97.5\%}$ $(7.46{\times}10^{-11})$
& $-0.146$ & $[-0.176,-0.116]$ \\
\midrule
\multirow{4}{*}{$\geq0$}
& Entropy
& $-0.290$ & $[-0.334,-0.245]$
& $\mathbf{-100\%}$ $(1.82{\times}10^{-12})$
& $-0.262$ & $[-0.304,-0.218]$ \\
& NLL
& $-0.244$ & $[-0.287,-0.201]$
& $\mathbf{-97.5\%}$ $(7.46{\times}10^{-11})$
& $-0.170$ & $[-0.206,-0.134]$ \\
& $p_{\mathrm{true}}$
& $+0.116$ & $[+0.084,+0.148]$
& $\mathbf{+87.5\%}$ $(1.38{\times}10^{-6})$
& $+0.170$ & $[+0.134,+0.206]$ \\
& Log-rank
& $-0.232$ & $[-0.265,-0.199]$
& $\mathbf{-97.5\%}$ $(7.46{\times}10^{-11})$
& $-0.159$ & $[-0.191,-0.127]$ \\
\bottomrule
\end{tabular}
\end{table}

\paragraph{Heterogeneity across model--book pairs.}
Table~\ref{tab:heterogeneity} reports the $I^2$ heterogeneity statistic for
each pooled Pearson correlation in Table~\ref{tab:predictive_full_stats}.
$I^2$ is moderate for NLL and $p_{\mathrm{true}}$ and higher for entropy,
indicating that the correlation with entropy varies more across model--book
pairs in magnitude, even though Table~\ref{tab:predictive_full_stats} shows
its \emph{sign} is consistent for all 40 pairs.

\begin{table}[h]
\centering
\small
\caption{Heterogeneity ($I^2$) of the pooled Pearson correlations reported in
Table~\ref{tab:predictive_full_stats}, by fit-quality filter.}
\label{tab:heterogeneity}
\begin{tabular}{@{}lcccc@{}}
\toprule
\textbf{$R^2$ filter} & Entropy & NLL & $p_{\mathrm{true}}$ & Log-rank \\
\midrule
$\geq0.95$ & $66.6\%$ & $37.3\%$ & $40.3\%$ & $24.7\%$ \\
$\geq0$    & $73.7\%$ & $69.7\%$ & $39.5\%$ & $48.9\%$ \\
\bottomrule
\end{tabular}
\end{table}

\paragraph{Robustness to the fit-quality filter.}
The $R^2\geq0.95$ filter excludes a different fraction of contexts in each
model--book pair; we call this the pair's \emph{rejection fraction} and ask
whether it is itself associated with that pair's average predictive
statistics, which would indicate the filter systematically selects easier or
harder contexts. Table~\ref{tab:rejection-fraction} shows none of the four
correlations is significant, so we find no evidence of such a bias.

\begin{table}[h]
\centering
\small
\caption{Pearson correlation between each model--book pair's $R^2\geq0.95$
rejection fraction and its mean predictive statistics ($N=40$ pairs). None
of the four associations is significant at the $0.05$ level.}
\label{tab:rejection-fraction}
\begin{tabular}{@{}lcc@{}}
\toprule
\textbf{Statistic} & $r$ & $p$ \\
\midrule
Entropy             & $0.227$  & $0.159$ \\
NLL                 & $0.210$  & $0.194$ \\
$p_{\mathrm{true}}$ & $-0.066$ & $0.684$ \\
Log-rank            & $0.189$  & $0.244$ \\
\bottomrule
\end{tabular}
\end{table}

\newpage
\section{Theoretical Characterization}
\label{appendix:theory-proofs}

This appendix collects the formal statements referenced in
Section~\ref{sec:results}. Section~\ref{subsec:mechanism} characterizes the
relationship between the log-linear logit law and an approximately Zipfian
conditional next-token distribution and clarifies what this relationship
does and does not establish. Sections~\ref{subsec:tail-count} and
\ref{subsec:zipf} derive two exact consequences of the empirical law, once it
holds for a given context: an equivalent exponential tail count
(Prop.~\ref{prop:tail-count}) and a conditional Zipf distribution with exact
temperature rescaling (Prop.~\ref{prop:zipf}), which together formalize the
assumption underlying Mirostat~\citep{basu2021mirostat}. Section~\ref{subsec:collapse-appendix}
formalizes the normalized-collapse behavior shown in
Figure~\ref{fig:logit_geometry}(b) (Prop.~\ref{prop:collapse}) and bounds its
robustness to residual fit error.

We distinguish two questions. The first concerns the consequences of the
vocabulary head being approximately log-linear in log-rank. These consequences
follow algebraically and do not depend on the origin of the observed law. The
second concerns why a language model's logits exhibit this form in the first
place. Our analysis provides only a partial account of this second question.

Fix a prediction context $h$ and write the vocabulary logits as
\[
    z_i(h)=w_i^\top h+b_i,
    \qquad i\in\mathcal V,
\]
where $w_i$ is the output embedding row for token $i$, $b_i$ is an output bias
when present, and $V=|\mathcal V|$. We suppress the dependence on $h$ when
clear, and write
\[
    z_{(1)}\ge z_{(2)}\ge\cdots\ge z_{(V)}
\]
for the logits sorted in descending order.

\subsection{Relation to a Zipfian conditional distribution}
\label{subsec:mechanism}

\paragraph{An exact equivalence at the cross-entropy optimum.}
At the population optimum of the cross-entropy objective, the softmax
probabilities for a fixed context $h$ equal the true conditional distribution
$q(\cdot\mid h)$, so the logits satisfy
\[
    z_i(h) = \log q(i\mid h) + c(h)
\]
for an arbitrary context-dependent constant $c(h)$ (adding a constant to
every logit of a context leaves the softmax unchanged). Sorting both sides by
rank,
\[
    z_{(r)}(h) = \log q_{(r)}(h) + c(h),
\]
where $q_{(r)}(h)$ is the $r$-th largest value of $q(\cdot\mid h)$. Thus, on
the vocabulary head, the empirical law
$z_{(r)}\approx\mu-\alpha\log r$ holds exactly when the conditional next-token
distribution is approximately Zipfian in rank,
$q_{(r)}(h)\propto r^{-\alpha}$. This is a change of variables rather than a
new assumption: log-linear logits and an approximately Zipfian
$q(\cdot\mid h)$ express the same claim, provided the model is near its
cross-entropy optimum.

\paragraph{What this equivalence does, and does not, explain.}
This identity does not by itself explain why $q(\cdot\mid h)$ should be close
to Zipfian; it instead transfers the same regularity from logits to
probabilities, where related structure has been conjectured and used
previously, most directly in Mirostat's online Zipf-exponent
estimate~\citep{basu2021mirostat}. Two points distinguish this from a purely
definitional restatement of what we measure. First, the identity above is
exact only at the unconstrained population optimum, while a finite model
trained on finite data with a fixed architecture and a low-rank output
embedding is not guaranteed to attain that optimum. The observation that
trained models realize this structure with $R^2\approx0.99$, with $\alpha$ confined
to a narrow band across four model families and an order of magnitude of
scale (Section~\ref{sec:results}), is therefore an empirical property of
trained models rather than a consequence of the training objective alone.
It is this tightness and stability, rather than the identity itself, that
requires explanation. Second, we measure the structure directly in logit
space, before any token is sampled and independently of any decoding rule,
which differs from the post-hoc statistics of generated text studied in
prior work on statistical laws of language
(Appendix~\ref{appendix:related-work}).

\paragraph{An equivalent distance-based parameterization.}
We also give an equivalent form, intended as an interpretation
rather than as a mechanistic model, in terms of a context-dependent mismatch, or
distance, $d_h(i)$ between the context and a candidate token, rather than
directly in terms of $q(\cdot\mid h)$. The two descriptions are equivalent by
construction: for any $\gamma_h>0$, writing
\[
    q(i\mid h)\propto\exp(-\gamma_h d_h(i))
\]
simply \emph{defines}
\[
    d_h(i):=-\gamma_h^{-1}\log q(i\mid h)+\mathrm{const}.
\]
Under this definition, the exponential-volume condition
\[
    G_h(t):=\#\{i\in\mathcal V:d_h(i)\le t\}
    \;\approx\; e^{\lambda_h t}
\]
is exactly the statement that $q(\cdot\mid h)$ is approximately Zipfian in
rank. We consider the resulting rank-distance relation as an alternative
interpretation of the observed log-linear law and to express temperature
rescaling and tail counts (Sections~\ref{subsec:tail-count} and
\ref{subsec:zipf}) in distance language as well. It is not intended as an
additional result beyond the equivalence above. Neither the exponential-volume
condition nor the linear-penalty form $z\approx\beta-\gamma d$ below is
independently verified in our data.

\begin{theorem}[rank-distance form of the log-linear law]
\label{thm:semantic-volume}
Fix a context $h$, and let $D_h(r)$ denote the $r$-th smallest value of
$\{d_h(i)\}_{i\in\mathcal V}$. Suppose that for ranks $1\le r\le K$,
\[
    \left|
    D_h(r)-\frac{\log r-\nu_h}{\lambda_h}
    \right|
    \le
    \frac{\eta_h}{\lambda_h},
    \qquad \lambda_h>0,
\]
and that the logits are uniformly close to a linear penalty in distance:
\[
    \max_{i\in\mathcal V}
    \left|z_i(h)-\bigl(\beta_h-\gamma_h d_h(i)\bigr)\right|
    \le \varepsilon_h,
    \qquad \gamma_h>0.
\]
Then, for $r=1,\ldots,K$,
\[
    \left|
    z_{(r)}(h)-\bigl(\mu_h-\alpha_h\log r\bigr)
    \right|
    \le
    \varepsilon_h+\alpha_h\eta_h,
\]
where
\[
    \alpha_h=\frac{\gamma_h}{\lambda_h},
    \qquad
    \mu_h=\beta_h+\alpha_h\nu_h.
\]
\end{theorem}

\begin{proof}
Define the ideal score
\[
    y_i(h)=\beta_h-\gamma_h d_h(i).
\]
Because $\gamma_h>0$, the ordering of $y_i(h)$ is the reverse of the ordering
of $d_h(i)$. Therefore, the $r$-th largest ideal score is
\[
    y_{(r)}(h)=\beta_h-\gamma_h D_h(r).
\]
Using the assumed rank-distance bound,
\[
    \left|
    y_{(r)}(h)-
    \left(
    \beta_h+\frac{\gamma_h\nu_h}{\lambda_h}
    -
    \frac{\gamma_h}{\lambda_h}\log r
    \right)
    \right|
    \le
    \frac{\gamma_h}{\lambda_h}\eta_h.
\]
With $\alpha_h=\gamma_h/\lambda_h$ and
$\mu_h=\beta_h+\alpha_h\nu_h$, this becomes
\[
    \left|
    y_{(r)}(h)-\bigl(\mu_h-\alpha_h\log r\bigr)
    \right|
    \le
    \alpha_h\eta_h.
\]
It remains to compare the order statistics of $z_i(h)$ and $y_i(h)$.
The uniform perturbation assumption gives
\[
    \max_{i\in\mathcal V}|z_i(h)-y_i(h)|\le\varepsilon_h.
\]
Order statistics are $1$-Lipschitz under the sup norm: if
$\max_i|a_i-b_i|\le\varepsilon$, then
$|a_{(r)}-b_{(r)}|\le\varepsilon$ for every rank $r$. Hence,
\[
    |z_{(r)}(h)-y_{(r)}(h)|\le\varepsilon_h.
\]
Combining the two inequalities yields
\[
    \left|
    z_{(r)}(h)-\bigl(\mu_h-\alpha_h\log r\bigr)
    \right|
    \le
    \varepsilon_h+\alpha_h\eta_h. \qedhere
\]
\end{proof}

The rank--distance bound in Theorem~\ref{thm:semantic-volume} is the
exponential-volume condition expressed in rank form. Define
\[
    G_h(t)=\#\{i\in\mathcal V:d_h(i)\le t\}.
\]
If, over the relevant head range,
\[
    \left|\log G_h(t)-\bigl(\nu_h+\lambda_h t\bigr)\right|\le\eta_h,
\]
then solving $G_h(t)\approx r$ gives
\[
    D_h(r)\approx\frac{\log r-\nu_h}{\lambda_h},
\]
which is the hypothesis of Theorem~\ref{thm:semantic-volume}. The next
paragraph establishes this implication directly. Under this parameterization,
the fitted slope decomposes as
\[
    \alpha_h
    =
    \frac{\gamma_h}{\lambda_h}
    =
    \frac{\textit{logit penalty per unit distance}}
    {\textit{distance-volume growth rate}},
\]
which is the same Zipf exponent $\alpha$ already defined in the main text,
expressed in terms of a penalty and a growth rate rather than as two
separately measurable quantities. Intuitively, $d_h(i)$ can be viewed as how well a candidate token fits what the
preceding context calls for. Within this parameterization, a larger $\alpha_h$
corresponds to a more selective context, concentrating probability on a smaller
set of plausible continuations and naturally yielding a sharper candidate set
under nucleus truncation.

\paragraph{Checking the relabeling.}
Let $G_h(t)=\#\{i\in\mathcal V:d_h(i)\le t\}$ and
$D_h(r)=\inf\{t:G_h(t)\ge r\}$. Assume that over the relevant distance range,
\[
    \left|\log G_h(t)-\bigl(\nu_h+\lambda_h t\bigr)\right|
    \le\eta_h,
    \qquad \lambda_h>0.
\]
For
\[
    t_+(r)=\frac{\log r-\nu_h+\eta_h}{\lambda_h},
\]
we have
\[
    \log G_h(t_+(r))
    \ge
    \nu_h+\lambda_h t_+(r)-\eta_h
    =
    \log r,
\]
so $G_h(t_+(r))\ge r$ and therefore $D_h(r)\le t_+(r)$. Similarly, for any
\[
    t<t_-(r)=\frac{\log r-\nu_h-\eta_h}{\lambda_h},
\]
we have
\[
    \log G_h(t)
    \le
    \nu_h+\lambda_h t+\eta_h
    <
    \log r,
\]
so $G_h(t)<r$ and hence $D_h(r)\ge t$. Letting
$t\uparrow t_-(r)$ gives $D_h(r)\ge t_-(r)$. Combining the two bounds,
\[
    \left|
    D_h(r)-\frac{\log r-\nu_h}{\lambda_h}
    \right|
    \le
    \frac{\eta_h}{\lambda_h}.
\]

\subsection{Tail-count form of the rank law}
\label{subsec:tail-count}

The preceding section concerns the possible origin of the empirical law.
Here we characterize its consequences for a fixed logit vector, independently
of its origin. Define the empirical upper-tail count
\[
    N(x)=\#\{i\in\mathcal V:z_i\ge x\}.
\]
For distinct top-$K$ logits, $N(z_{(r)})=r$.

\begin{proposition}[tail-count characterization]
\label{prop:tail-count}
Assume the top-$K$ logits are distinct and
\[
    z_{(r)}=\mu-\alpha\log r+\varepsilon_r,
    \qquad \alpha>0,\quad r=1,\ldots,K.
\]
Then, at $x=z_{(r)}$,
\[
    \log N(x)
    =
    \frac{\mu-x}{\alpha}
    +
    \frac{\varepsilon_r}{\alpha}.
\]
Conversely, if for all $r\le K$,
\[
    \left|
    \log N(z_{(r)})-\frac{\mu-z_{(r)}}{\alpha}
    \right|
    \le \eta,
\]
then
\[
    \left|z_{(r)}-\bigl(\mu-\alpha\log r\bigr)\right|
    \le \alpha\eta.
\]
\end{proposition}

\begin{proof}
For distinct top-$K$ logits, $N(z_{(r)})=r$. If
$z_{(r)}=\mu-\alpha\log r+\varepsilon_r$, then
\[
    \log r
    =
    \frac{\mu-z_{(r)}}{\alpha}
    +
    \frac{\varepsilon_r}{\alpha}.
\]
Substituting $\log N(z_{(r)})=\log r$ gives
\[
    \log N(z_{(r)})
    =
    \frac{\mu-z_{(r)}}{\alpha}
    +
    \frac{\varepsilon_r}{\alpha}.
\]
Conversely, if
$\left|\log N(z_{(r)})-\frac{\mu-z_{(r)}}{\alpha}\right|\le\eta$, then using
$N(z_{(r)})=r$ gives
\[
    \left|
    \log r-\frac{\mu-z_{(r)}}{\alpha}
    \right|\le\eta.
\]
Multiplying by $\alpha$ and rearranging gives
\[
    \left|z_{(r)}-\bigl(\mu-\alpha\log r\bigr)\right|
    \le\alpha\eta. \qedhere
\]
\end{proof}

Thus, on the vocabulary head, the log-linear law is equivalent to an
approximately exponential empirical upper tail:
\[
    N(x)\approx \exp\!\left(\frac{\mu-x}{\alpha}\right).
\]
This statement is deterministic for each context and does not require treating
the logits as independent random variables.

\subsection{Softmax consequence: a conditional Zipf law}
\label{subsec:zipf}

The log-linear law has a direct consequence for the model's predictive
distribution. Let
\[
    p_{(r)}=
    \frac{\exp z_{(r)}}{\sum_{s=1}^{V}\exp z_{(s)}}
\]
be the full softmax probability of the token at rank $r$. Since the empirical
law is fitted over the top $K$ logits, define the top-$K$ probability mass
\[
    M_K=\sum_{s=1}^{K}p_{(s)}
\]
and the conditional distribution within the vocabulary head
\[
    q_r^{(K)}=
    \frac{p_{(r)}}{M_K}
    =
    \frac{\exp z_{(r)}}{\sum_{s=1}^{K}\exp z_{(s)}}.
\]

\begin{proposition}[conditional Zipf law]
\label{prop:zipf}
Suppose
\[
    z_{(r)}=\mu-\alpha\log r+\varepsilon_r,
    \qquad |\varepsilon_r|\le\delta,\quad r=1,\ldots,K.
\]
Let
\[
    H_{K,\alpha}=\sum_{s=1}^{K}s^{-\alpha}.
\]
Then
\[
    e^{-2\delta}\frac{r^{-\alpha}}{H_{K,\alpha}}
    \le
    q_r^{(K)}
    \le
    e^{2\delta}\frac{r^{-\alpha}}{H_{K,\alpha}}.
\]
Equivalently,
\[
    M_K e^{-2\delta}\frac{r^{-\alpha}}{H_{K,\alpha}}
    \le
    p_{(r)}
    \le
    M_K e^{2\delta}\frac{r^{-\alpha}}{H_{K,\alpha}}.
\]
\end{proposition}

\begin{proof}
For each $r\le K$,
\[
    \exp z_{(r)}=e^\mu r^{-\alpha}e^{\varepsilon_r}.
\]
Since $|\varepsilon_r|\le\delta$,
\[
    e^\mu e^{-\delta}r^{-\alpha}
    \le
    \exp z_{(r)}
    \le
    e^\mu e^\delta r^{-\alpha}.
\]
Summing over $s=1,\ldots,K$ gives
\[
    e^\mu e^{-\delta}H_{K,\alpha}
    \le
    \sum_{s=1}^{K}\exp z_{(s)}
    \le
    e^\mu e^\delta H_{K,\alpha}.
\]
Dividing the pointwise bounds by the denominator bounds yields
\[
    e^{-2\delta}\frac{r^{-\alpha}}{H_{K,\alpha}}
    \le
    q_r^{(K)}
    \le
    e^{2\delta}\frac{r^{-\alpha}}{H_{K,\alpha}}.
\]
The bound for $p_{(r)}$ follows from
$p_{(r)}=M_Kq_r^{(K)}$. \qedhere
\end{proof}

Therefore, the log-linear logit law implies an approximate Zipf distribution
for the model's context-dependent next-token probabilities, restricted to the
vocabulary head. This differs from corpus-level Zipf laws: here the
distribution is produced by the model before sampling and varies with the
context.

Temperature scaling follows directly. At temperature $T>0$,
\[
    \frac{z_{(r)}}{T}
    =
    \frac{\mu}{T}
    -
    \frac{\alpha}{T}\log r
    +
    \frac{\varepsilon_r}{T}.
\]
Thus, in the exact case, temperature rescales the exponent as
\[
    \alpha_T=\frac{\alpha}{T}.
\]
Lower temperatures sharpen the head distribution, while higher temperatures
flatten it.

\subsection{Normalized spectrum collapse}
\label{subsec:collapse-appendix}

Figure~\ref{fig:logit_geometry}(b) shows this collapse, and
Section~\ref{sec:results} describes it empirically; here we state it
formally. The collapse is a direct algebraic consequence of $z$-scoring an
affine function of $\log r$: any two exactly log-linear spectra collapse onto
the same normalized curve regardless of their $\mu$ and $\alpha$. Thus, the
collapse itself provides no evidence for cross-context universality beyond the
quality of the affine fit in each context. The empirical claim is instead
that individual contexts are well described by that affine form with a stable
slope (Section~\ref{sec:results}). The stability result below quantifies how
far the normalized collapse can deviate from its idealized form when the
log-linear fit is approximate.

\begin{proposition}[normalized spectrum collapse]
\label{prop:collapse}
If $z_{(r)} = \mu - \alpha \log r$ for $r = 1,\ldots,K$, then the standardized
spectrum $\tilde z_{(r)} = (z_{(r)} - \bar z_K)/\sigma_{z,K}$ satisfies
\[
    \tilde z_{(r)} = \frac{\bar\ell_K - \log r}{s_{\ell,K}},
\]
where $\ell_r = \log r$, and $\bar\ell_K$, $s_{\ell,K}$ are the mean and
standard deviation of $\{\ell_r\}_{r=1}^K$, independent of $\mu$ and
$\alpha$.
\end{proposition}

\begin{proof}
Let $\ell_r=\log r$. If
\[
    z_{(r)}=\mu-\alpha\ell_r,
\]
then
\[
    \bar z_K=\mu-\alpha\bar\ell_K.
\]
Therefore
\[
    z_{(r)}-\bar z_K=\alpha(\bar\ell_K-\ell_r).
\]
Also,
\[
    \sigma_{z,K}^2
    =
    \frac1K\sum_{r=1}^{K}\alpha^2(\ell_r-\bar\ell_K)^2
    =
    \alpha^2s_{\ell,K}^2.
\]
Since $\alpha>0$, $\sigma_{z,K}=\alpha s_{\ell,K}$. Hence
\[
    \widetilde z_{(r)}
    =
    \frac{z_{(r)}-\bar z_K}{\sigma_{z,K}}
    =
    \frac{\bar\ell_K-\log r}{s_{\ell,K}}. \qedhere
\]
\end{proof}

\paragraph{Stability of the normalized collapse.}
In real contexts, the log-linear law is only approximate, so we consider the
effect of residual fit error on the normalized spectrum. Let
\[
    z_{(r)}=\mu-\alpha\ell_r+\varepsilon_r,
\]
and define the empirical norm
\[
    \|x\|_K=
    \left(\frac1K\sum_{r=1}^{K}x_r^2\right)^{1/2}.
\]
Let
\[
    a_r=\bar\ell_K-\ell_r,
    \qquad
    \phi_K(r)=\frac{a_r}{\|a\|_K},
\]
so $\|a\|_K=s_{\ell,K}$. Centering removes $\mu$ and the mean residual. With
\[
    e_r=\varepsilon_r-\frac1K\sum_{s=1}^{K}\varepsilon_s,
\]
the centered spectrum is
\[
    z_{(r)}-\bar z_K=\alpha a_r+e_r.
\]
If
\[
    \rho=\frac{\|e\|_K}{\alpha s_{\ell,K}}<1,
\]
then by the reverse triangle inequality,
\[
    \|\alpha a+e\|_K\ge \alpha s_{\ell,K}(1-\rho).
\]
Using
\[
    \left\|\frac{u+v}{\|u+v\|}
    -\frac{u}{\|u\|}\right\|
    \le
    \frac{2\|v\|}{\|u+v\|}
\]
with $u=\alpha a$ and $v=e$ gives
\[
    \left\|\widetilde z-\phi_K\right\|_K
    \le
    \frac{2\rho}{1-\rho}.
\]
If $|\varepsilon_r|\le\delta$ for all $r\le K$, then
$\|e\|_K\le\delta$, so whenever $\delta<\alpha s_{\ell,K}$,
\[
    \left\|\widetilde z-\phi_K\right\|_K
    \le
    \frac{2\delta}{\alpha s_{\ell,K}-\delta}.
\]

\end{document}